\documentclass[letterpaper, 10 pt, conference]{ieeeconf}  % Comment this line out if you need a4paper

\IEEEoverridecommandlockouts                              % This command is only needed if 
\pdfoutput=1 
\usepackage{graphicx}
\title{GraphRQI: Classifying Driver Behaviors Using Graph Spectrums}
\author{Rohan Chandra$^1$, Uttaran Bhattacharya$^1$, Trisha Mittal$^1$, Xiaoyu Li$^2$, Aniket Bera$^1$ and Dinesh Manocha$^1$\\%
$^1$University of Maryland $^2$Cornell University\\
\small{Code and Supplementary Material at \url{https://gamma.umd.edu/graphrqi}}
}

\newcommand{\bigO}[1]{\mathcal{O}(#1)}

\newcommand{\algoname}{GraphRQI}

\newcommand{\sota}{state-of-the-art}

\newcommand{\brr}[1]{\left( #1 \right)}

\newcommand{\bss}[1]{\left[ #1 \right]}
\newcommand{\mc}[1]{\mathcal{#1}}
\newcommand{\sg}{\mathcal{L}}
\newcommand{\li}{\sg}
\newcommand{\vts}[1]{\lvert #1 \rvert}
\newcommand{\Vts}[1]{\lVert #1 \rVert}
\newcommand{\bb}[1]{\mathbb{#1}}
\newcommand\Tstrut{\rule{0pt}{2.6ex}}         % = `top' strut
\newcommand\Bstrut{\rule[-1.3ex]{0pt}{0pt}}   % = `bottom' strut

\newcommand{\cost}{\bigO{ \vts{\li^{\scriptscriptstyle -1}_t}k }}
\newcommand{\costk}{\bigO{ \vts{\li^{\scriptscriptstyle -1}_t} }}
\makeatletter
\newcommand\footnoteref[1]{\protected@xdef\@thefnmark{\ref{#1}}\@footnotemark}
\makeatother
\newcommand{\size}{\bigO{d}}
\newcommand{\shorteq}{%
  \settowidth{\@tempdima}{-}% Width of hyphen
  \resizebox{\@tempdima}{\height}{=}%
}

\usepackage{amsmath}
\usepackage[utf8]{inputenc}
\usepackage[T1,OT1]{fontenc}
\usepackage[english]{babel}
\usepackage{xcolor}
\usepackage{algorithmicx}
\usepackage[Algorithm,ruled]{algorithm}
\usepackage[font=small,labelfont=bf]{caption} % This causes problems while compiling to latex    
\usepackage[noend]{algpseudocode}
\usepackage{array}
\usepackage{graphicx}
\usepackage{amsfonts}
\usepackage{hhline}
\usepackage{multirow, makecell}
\usepackage{float}
\usepackage{booktabs}

\usepackage{amsthm}
\usepackage{color}
\usepackage{transparent}
\usepackage{url}
\usepackage[symbol]{footmisc}
\usepackage{setspace}

\theoremstyle{plain}

\newtheorem{theorem}{Theorem}[section]

\newtheorem{lemma}[theorem]{Lemma}

\newtheorem*{thm}{Theorem}
\newtheorem*{lem}{Lemma}
\begin{document}

\maketitle
\thispagestyle{empty}
\pagestyle{empty}

\begin{abstract}
\label{sec: abstract}
We present a novel algorithm (GraphRQI) to identify driver behaviors from road-agent trajectories. Our approach assumes that the road-agents exhibit a range of driving traits, such as aggressive or conservative driving. Moreover, these traits affect the trajectories of nearby road-agents as well as the interactions between road-agents. We represent these inter-agent interactions using unweighted and undirected traffic graphs. Our algorithm classifies the driver behavior using a supervised learning algorithm by reducing the computation to the spectral analysis of the traffic graph. Moreover, we present a novel eigenvalue algorithm to compute the spectrum efficiently. We provide theoretical guarantees for the running time complexity of our eigenvalue algorithm and show that it is faster than previous methods by 2 times. We evaluate the classification accuracy of our approach on traffic videos and autonomous driving datasets corresponding to urban traffic. In practice, GraphRQI achieves an accuracy improvement of up to $25\%$ over prior driver behavior classification algorithms. We also use our classification algorithm to predict the future trajectories of road-agents. 

\end{abstract}

\section{Introduction}
\label{sec: intro}
Autonomous driving is an active area of research~\cite{ad1,ad2} and includes many issues related to perception, navigation, and control~\cite{ad2}. While some of the earlier work was focused on autonomous driving in mostly empty roads or highways, there is considerable recent work on developing technologies in urban scenarios. These urban scenarios consist of a variety of road-agents corresponding to vehicles, buses, trucks, bicycles, and pedestrians. In order to perform collision-free navigation, it is important to track the location of these road-agents~\cite{chandra2019roadtrack,chandra2019densepeds} and predict their trajectory in the near future~\cite{pendleton2017perception,schwarting2018planning,social-lstm,social-gan,chandra2019traphic,chandra2019robusttp,chandra2019forecasting}. 
% While autonomous vehicles can work relatively well in sparse conditions where road-agents follow strict traffic rules, autonomous navigation is challenging in dense and heterogeneous traffic where agents may not follow standard traffic rules such as observing speed limits, lane-driving, or right-of-way. 
These tasks can be improved with a better understanding and prediction of driver behavior and performing behavior-aware planning~\cite{ernest,ernest_nav,schwarting2018planning,paden2016survey}. For example, drivers can navigate dangerous traffic situations by identifying the aggressive drivers and avoiding them or keeping a distance. At the same time, there is a clear distinction between such driver behaviors and maneuver-based road-agent behaviors~\cite{pseudobehavior1,honda,pseudobehavior2,pseudobehavior3}. According to studies in traffic psychology~\cite{humanfactor1,humanfactor2,humanfactor3}, driver behavior corresponds to labels for various driving traits, such as aggressive and conservative driving~\cite{humanfactor1,humanfactor2,humanfactor3}. These studies further show that traffic flow and road-agent trajectories are significantly affected by driver behaviors. For instance, an aggressive road-agent overtaking a vehicle may cause the vehicle to slow down, or an overly cautious driver may drive slowly and cause congestion.
% It makes sense, then, that a successful driver behavior prediction system should capture the among nearby road-agents. 

Our goal is to classify the drivers according to their behaviors in a given traffic video. The two most widely used behavior labels used to classify drivers are aggressive and conservative~\cite{rohanref3,rohanref4,rohanref5}. Recent studies have also explored more granular behaviors, such as reckless, impatient, careful and timid~\cite{rohanref6-reckless, rohanref7-timid}. For instance, impatient drivers show different traits (e.g., over-speeding, overtaking) than reckless drivers (e.g., driving in the opposite direction to the traffic)~\cite{rohanref6-reckless}, but both can be regarded as aggressive. Similarly, timid drivers may drive slower than the speed limit~\cite{rohanref7-timid}, while careful drivers obey the traffic rules. However, both of them are marked as conservative drivers. In this paper, we consider a total of six different classes of driver behavior --- impatient, reckless, threatening, careful, cautious, and timid~\cite{ernest}. For the rest of the paper, we collectively refer to the first three as aggressive driving, while the last three are categorized as part of conservative driving.

\begin{figure}[t]
    \centering
    \includegraphics[scale=.28, width=\columnwidth]{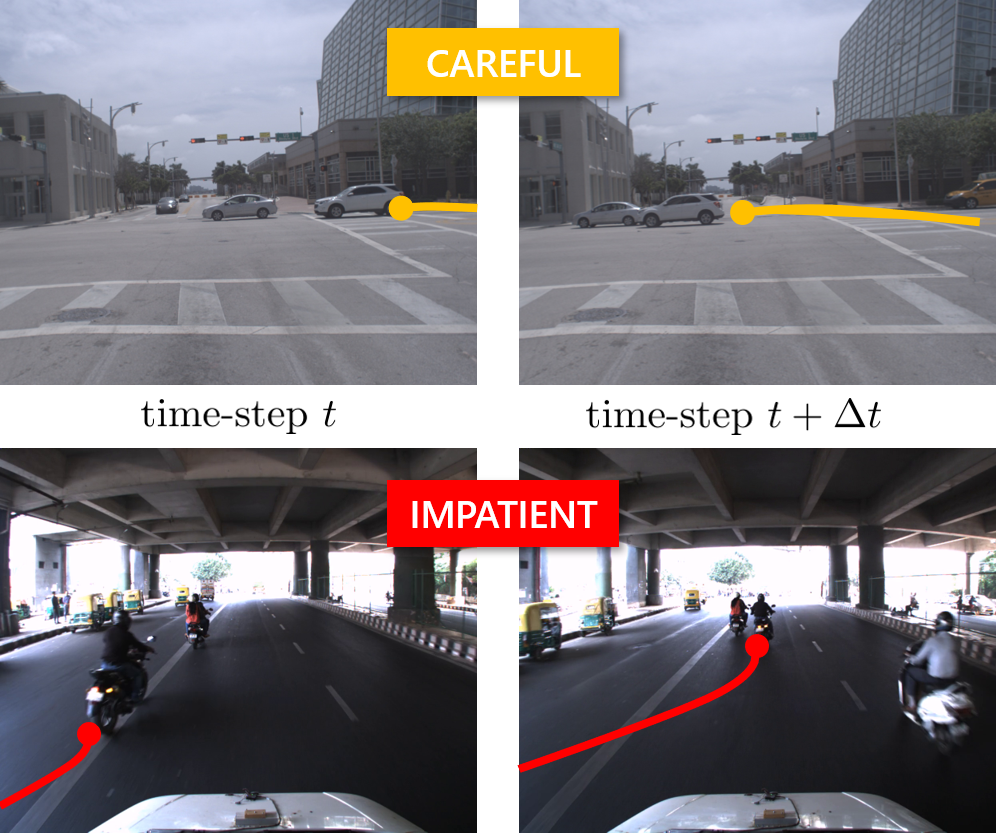}
    \caption{GraphRQI predicts the behaviors of road-agents in dense and heterogeneous traffic from one of the following classes--\textit{impatient, reckless, threatening, careful, cautious, timid}. Our approach is up to 25\% more accurate than prior behavior prediction methods.}
    \label{fig:cover0}
    \vspace{-10pt}
\end{figure}

\begin{figure*}[h]
    \centering
    % \resizebox{.7}{!}{
    \includegraphics[width=\textwidth]{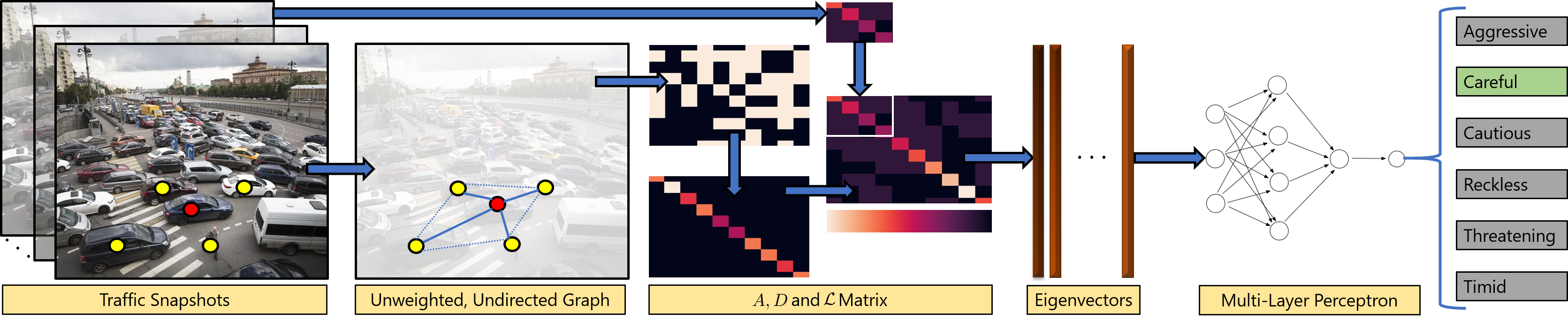}
    % }
    \caption{\textbf{Overview of GraphRQI: }At any instance of traffic, we use the current trajectories to form an unweighted and undirected graphical representation. We can then compute the adjacency matrix, the degree matrix, and the Laplacian matrix of the graph. Note that the principal leading submatrix of the Laplacian can be obtained from the Laplacian of the traffic from previous time-steps. We use our eigenvalue algorithm to compute the spectrum of the Laplacian, which is then used for training via a multi-layer perceptron to predict the driver behavior.}
    \label{fig:cover}
    \vspace{-10pt}
\end{figure*}
% Graphs encode the interactions between objects through edges formed between the vertices.
Since the behavior of individual drivers directly influences the trajectories of all other drivers in their neighborhood, we can represent the drivers as the nodes of a sparse, undirected, and unweighted graph, and the interactions between the drivers through edges between the corresponding nodes. In our formulation, we analyze the spectral properties of such a graph for behavior classification. However, there are a few challenges in calculating the spectrum of such graphs. First, traffic flow is dynamic, therefore the number of edges in the graph changes at each time-step. Second, the number of nodes in the graph tends to increase monotonically with each time-step, since retaining information on past interactions is useful for predicting current and future behaviors of other drivers~\cite{social-lstm}. This is a problem for current \sota~methods that perform graph spectral analysis ~\cite{cg}, as the computational complexity grows at least as $\bigO{n^{1.5}}$, where $n$ is the number of nodes in the graph. In order to classify the behaviors, we need to design efficient methods to perform spectral analysis of such traffic graphs.

\noindent{\bf Main Results:} We present a novel algorithm to classify driver behaviors from road-agent trajectories. Our approach is general, and we don't make any assumption about traffic density, heterogeneity, or the underlying scenarios. The trajectories are extracted from traffic videos or other sensor data, and the resulting behaviors are classified among the three aggressive or three conservative behaviors described above. In order to efficiently classify the behavior, we present two novel algorithms:
\begin{enumerate}
    \item A classification algorithm, \algoname, that takes as input road-agent trajectories, converts them into a dynamic-graph-based representation of the traffic, and predicts driver behaviors as one of six labels from the spectrum of the dynamic graph. 
    
    \item An eigenvalue algorithm that updates the spectrum of the dynamic traffic graph with $n$ current nodes in $\bigO{n}$ time, up from $\bigO{n^{1.5}}$ required by prior methods. 
    
    % GraphRQI is an iterative algorithm that extends the classical RQI algorithm and converges cubically to an eigenvector-eigenvalue pair. Our algorithm does not require any pre-conditioning initialization, sparse-matrix approximations, or low-rank matrix approximations.
    % Our algorithm converges cubically (distance from groundtruth = $\bigO{\epsilon^3}, \epsilon \ll 1$) at every iteration, enabling realtime driver behavior prediction.
    %\item We demonstrate an application where we show that driver behavior can be used for predicting the trajectories of road-agents.

\end{enumerate}
Finally, we demonstrate an application where knowledge of the behavior of a road-agent allows nearby road-agents to generate efficient and safer paths based on different behavioral characteristics. We have compared our algorithm with prior classification schemes and observe the following benefits:
\begin{itemize}

    \item \textit{Accuracy improvement:} We report improvement of up to $20\%$ and $25\%$ in the weighted classification accuracy on the TRAF dataset~\cite{chandra2019traphic} and the Argoverse dataset~\cite{argo}, respectively.
    
    \item \textit{Speed improvement:} With up to $100$ drivers, the runtime of our eigenvalue algorithm is $10$ milliseconds which is a speed of up to $2.348$ seconds over prior works for calculating the spectrums of dynamic traffic graph.
    % All the experiments were run on an 8 Core Intel Xeon(R) W2123 CPU clocked at 3.60GHz with 32 GB RAM.
    
    % \item Robustness to sensor noise: Our algorithm is more robust to sensor noise than prior methods and can predict driver behavior from videos captured from the front camera as well as top-down cameras. 
\end{itemize}
\vspace*{-5pt}
\section{Related Work}
\label{sec: related}
\subsection{Driver Behavior and Intent Prediction}

There is a large body of research on driver behavior and intent prediction. Furthermore, many studies have been performed that provides insights into factors that contribute to different driver behavior.
%Due to readily available biographical information of human beings such as age, gender, occupation, and so on, many studies have been performed that provides insights into factors that contribute to different driver behavior.
Feng et al.~\cite{ernestref2} proposed five driver characteristics (age, gender, personality via blood test, and education level) and four environmental factors (weather, traffic situation, quality of road infrastructure, and other cars’ behavior), and mapped them to 3 levels of aggressiveness (driving safely, verbally abusing other drivers, and taking action against other drivers). Rong et al.~\cite{rohanref3} presented a similar study but instead used different features such as blood pressure, hearing, and driving experience to conclude that aggressive drivers tailgate and weave in and out of traffic. Dahlen et al.~\cite{rohanref5} studied the relationships between driver personality and aggressive driving using the five-factor-model~\cite{big5}. Aljaafreh et al.~\cite{ernestref8} categorized driving behaviors into four classes: Below normal, Normal, Aggressive, and Very aggressive, in accordance with accelerometer data. Social Psychology studies~\cite{ernestref9,ernestref10} have examined the aggressiveness according to the background of the driver, including age, gender, violation records, power of cars, occupation, etc. Mouloua et al.~\cite{ernestref11} designed a questionnaire on subjects’ previous aggressive driving behavior and concluded that these drivers also repeated those behaviors under a simulated environment. However, the driver features used by these methods cannot be computed easily for autonomous driving in new and unknown environments, which mainly rely on visual sensors.
%using current sensors like cameras or lidars. 

Several methods have analyzed driver behavior using visual and other information. Murphey et al.~\cite{ernestref12} conducted an analysis on the aggressiveness of drivers and observed that longitudinal (changing lanes) jerk is more related to aggressiveness than progressive (along the lane) jerk (i.e., rate of change in acceleration). Mohamad et al.~\cite{ernestref13} detected abnormal driving styles using speed, acceleration, and steering wheel movement, which indicated the direction of vehicles. Qi et al.~\cite{ernestref14} studied driving styles with respect to speed and acceleration. Shi et al.~\cite{ernestref15} pointed out that deceleration is not very indicative of the aggressiveness of drivers, but measurements of throttle opening, which are associated with acceleration, were more helpful in identifying aggressive drivers. Wang et al.~\cite{ernestref16} classified drivers into two categories, aggressive and normal, using speed and throttle opening captured by a simulator. Sadigh et al.~\cite{ernestref17} proposed a data-driven model based on Convex Markov Chains to predict whether a driver is paying attention while driving. Some recent approaches also utilize RL (reinforcement learning) and imitation learning techniques to teach networks to recognize various driver intentions~\cite{qi2018intent,codevilla2018end}, while others have applied RNN and LSTM-based networks to perform data-driven driver intent prediction~\cite{zyner2018recurrent,zyner2017long}. Apart from behavior prediction, some methods have been proposed for behavior modeling~\cite{yeh2008composite,helbing1995social,bera2016glmp,bera2016realtime, guy2012statistical}. In our case, we exploit the spectrum of dynamic traffic graphs to predict driver behaviors, which is complementary to these approaches and can be combined with such methods.
% are analytic studies that are useful in determining factors that contribute to different behaviors, they do not predict driver behavior on the roads.
\vspace*{-7pt}
\subsection{Supervised Learning Using Graphs}

% Supervised learning is a popular paradigm of machine learning. 
% Algorithms using Convolutional Neural Networks (CNNs) have shown to be the state-of-the-art in classification and prediction tasks on one-, two-, and three-dimensional grid-like data such as images, video, speech, and text~\cite{cnn1,cnn2}.
% However, some applications use data that do not have an underlying grid structure, for instance, social networks. 
Bruna et al.~\cite{bruna2013spectral} introduce supervised learning using convolutions on graphs via graph convolutional networks (GCNs). Kipf et al.~\cite{kipf2016semi} extend GCNs to self-supervised learning. However, these methods are restricted to static graphs. In order to model the temporal nature of traffic using graphs, Yu et al.~\cite{stgcn} used a Spatio-temporal GCN (STGCN). The authors use recurrent neural networks to compute the temporal features of the dynamic traffic network.
% This is useful for discovering patterns in the data along the time axis. However, for predicting driver behavior, we are interested in extracting patterns in the graph structure itself, using spectral theory on dynamic graphs. Computing the spectrum for dynamic graphs is still an active area of research. 
TIMERS~\cite{zhang2018timers} incrementally computes the spectrum using incremental Singular Value Decomposition (SVD)~\cite{incsvd}. The method implemented a restart on the SVD when the error increases above a threshold. Li et al.~\cite{li2017attributed} proposed an algorithm for computing the spectrum for node classification tasks that support addition and deletions of edges. 

\section{Background}
\label{sec: background}

\begin{table}
  \caption{Notations used in the paper. }
  \label{tab: notation}
  \centering
  \resizebox{.9\columnwidth}{!}{%
  \begin{tabular}{|c|l|}
  \hline
    Symbol  & Description \\
    \hline
    $n$ & Number of vehicles\\
    $[m]$ & \{1,2,\ldots,m\} \\
    $\vts{\cdot}$ & Number of non-zero elements in a vector or matrix.\\
    $\mc{I}_d$& Set of indices corresponding to non-zero entries in a vector $d$.\\
    $A_\mc{S}$& Set of columns of $A$, indexed by elements of $\mc{S}$.\\
    $\li_t$ & The $t^{\textrm{th}}$ Laplacian matrix in a sequence of sub-degree matrices. \\
    $d = \textrm{dim}\li_t$ & The dimension number of rows (or columns)) of the laplacian matrix .\\
    $\delta$ & A binary vector of length $n$ such that $\vts{\delta} \ll n$ \\
    $U$ & Eigenvector matrix\\
    $\Lambda$ & Diagonal matrix with eigenvalues on the diagonal.\\
    \hline
  \end{tabular}
%   }
  }
  \vspace{-10pt}
\end{table}

\begin{table*}[h]
\caption{Analytical Comparison of GraphRQI with classical and state-of-the-art eigenvector algorithms. In practice, our observed runtime is $10$ milliseconds which is a speed of up to $2.348$ seconds over prior works for calculating the spectrums of dynamic traffic graph.}
\centering
\resizebox{\textwidth}{!}{%
\begin{tabular}{ccccccc} 
\toprule
%  & \multicolumn{6}{c}{Eigenvector Algorithms} \\
% \cmidrule{2-6}
% \cmidrule{9-11}
Metric & \Bstrut SVD & IncrementalSVD~\cite{incsvd} & RestartSVD~\cite{zhang2018timers}  & TruncatedSVD & CG~\cite{cg} &  \textbf{GraphRQI} \\
\midrule
% \multirow{4}{*}{TRAF }
        Time \Tstrut & $\bigO{d^3}$ & $\bigO{\vts{\li_t}r}$ & $\mathcal{O} \Big ( \vts{f(\li_t)} + \vts{g(\li_t)}k +\vts{g(\li_t)_{\overline{\emptyset}}}k^2 \Big )$ & $\bigO{dk^2}$ &$\bigO{d\sqrt{\kappa}}$& $\cost$\\
%      &&&&&& \\
%               &&&&&& \\
% % \cline{2-8}
%               &&&&&& \\
% \midrule
% \multirow{4}{*}{ARGO }
        Space \Tstrut & $\bigO{d^2}$ & $\bigO{2 \ dr}$ &-& $\bigO{dk}$ & $\bigO{d}$ & $\size$\\
%       &&&&&& \\
%               &&&&&& \\
% % \cli_ine{2-8}
%              &&&&&& \\
\bottomrule
\end{tabular}
}
\label{tab:bigo}
 \vspace{-10pt}
\end{table*}
% \vspace*{-12pt}

Our overall objective is to classify all drivers in a traffic video into one of six behavior classes --- impatient, reckless, threatening, careful, cautious, and timid. We assume that the trajectories of all the road-agents in the video are provided to us as the input. Given this input, we first construct a traffic graph at each time-step, with the number of road-agents during that time-step being the number of nodes and undirected, unweighted edges connecting the neighboring road-agents during that time-step. Next, we show that the spectrum of the Laplacian matrix of the graphs contains pertinent information to classify the nodes (drivers) into corresponding behavior classes (Section IV).

In this section, we give an overview of a graph spectrum and a brief overview of graph topology. Notations used frequently throughout the paper are presented in Table~\ref{tab: notation}.
%Notations used less frequently are defined insitu.
% Our problem statement is as follows: At any time instance $i$, given only the trajectory of all road-agents, our goal is to predict the behavior of the driver. We present the main algorithm, \algoname~in Section~\ref{sec: algorithm}. The major step in \algoname~is to form a graph from the trajectories of the road-agents and to compute the spectrum of the graph using spectral graph theory. We give a brief background about spectral graph theory below. 

\begin{figure}[t]
% \begin{tabular}{ll}
\includegraphics[width = .95\columnwidth]{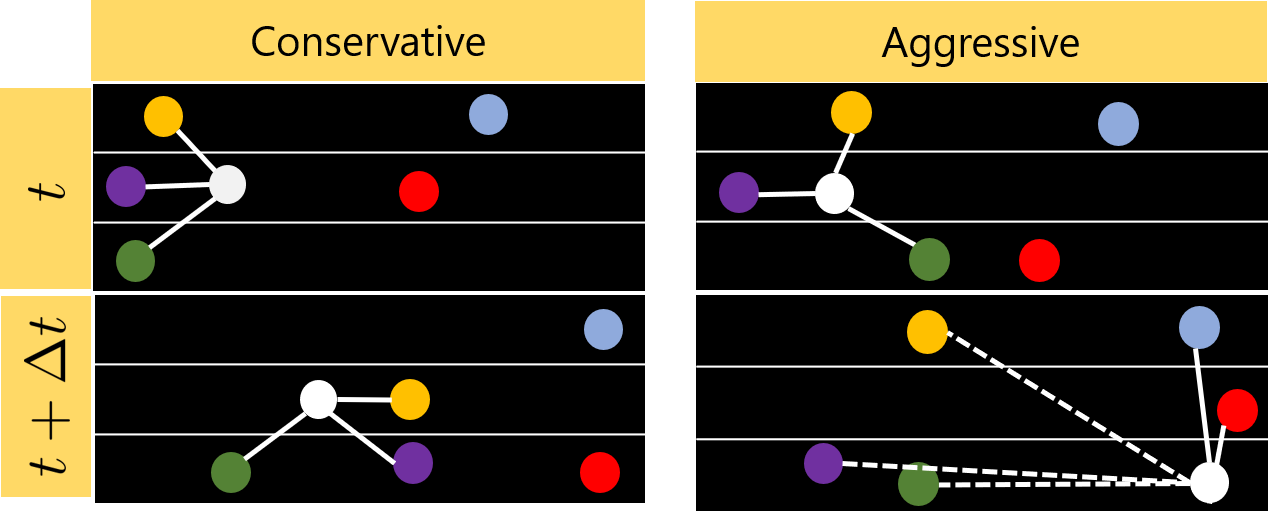}
% &
% \includegraphics[width = \columnwidth]{figs/agg-resize.png}
% \end{tabular}
 \caption{(\textit{left}) Conservative vehicles tend to adhere to a uniform speed. Consequently, they form a uniform neighborhood that results in smooth topology in the traffic graph with short edges. (\textit{right}) Aggressive drivers tend to display maneuvers such as tailgating, over-speeding, overtaking, and weaving. These maneuvers cause the road-agent to form sharper subgraphs in the traffic graph with long edges. }
% (right) The eigenvector for road-agents at any given time-instant is plotted with road-agents on the x axis and the value of each entry of the eigenvector on the y axis. The $i^{\textrm{th}}$ entry corresponds to the $i^{\textrm{th}}$ road-agent. High entries correspond to sharp graph topologies (aggressive (impatient, threatening, reckless) road-agents) while low entries correspond to smooth graph topologies (conservative (timid, cautious, careful)).}
\label{fig: main}
\vspace{-15pt}
\end{figure}
% Spectral graph theory is a broad field of computer science that deals with graph-based data. It has numerous applications such as graph visualizations, spectral clustering, and graph colorings. For a detailed review of spectral graph theory, we refer the reader to XXXX. The properties and structure of graphs are studied with the help of matrices. For example, three matrices are typically constructed: the adjacency matrix, $A$, the diagonal matrix, $D$, and the Laplacian, $L = D - A$. These are defined as follows.

\subsection{Graph Spectrum}
\label{subsec: graph_spectrum}

Given a graph $\mc{G}$, with a set of vertices $\mc{V} = \{ v_1, v_2, \dots, v_n \}$ and a set of undirected, weighted edges $e_{ij} = e_{ji} = (v_i, v_j)$, $\mc{E}$, we define the entries of  its adjacency matrix $A \in \mathbb{R}^{n \times n}$ as $A(i, j) = A(j, i) = e^{-d(v_i,v_j)}$ if $e_{ij} \in \mc{E}$ and $A(i, j) = 0$, otherwise.  The function $f(v_i,v_j) = e^{-d(v_i,v_j)}$~\cite{belkin2003laplacian} denotes the interactions between any two road-agents, $v_i$ and $v_j$. The corresponding degree matrix $D \in \mathbb{R}^{n \times n}$ of the graph is a diagonal matrix with entries $D(i, i) = \sum_{j=1}^n A(i, j)$. The Laplacian matrix $\mc{L}  = D  - A$ of the graph has the entries,
%. Let $\mc{G}= (\mc{V}, \mc{E})$, with $\lvert \mc{V} \rvert = n$, then the Laplacian matrix associated with $\mc{G}$ is an $n\times n$ matrix $\mc{L} = D-A$, where $D \in \mathbb{R}^{n\times n}$ is the degree matrix - a diagonal matrix where $D(i, i)$ is the degree of the $i^{th}$ node in $\mc{G}$, and $A \in \mathbb{R}^{n\times n}$ is the adjacency matrix, with $A(i, j) = 1$ if and only if $(v_i, v_j) \in \mc{E}$. We have then,
\vspace*{-5pt}
\begin{equation*}
\mc{L}(i,j) = \mc{L}(j,i) = 
     \begin{cases}
       D(i,i)  &\text{if $i=j$},\\
      -e^{-d(v_i,v_j)} &\text{if $d(v_i,v_j) < \mu$}, \\
      0 &\text{otherwise.}
        
     \end{cases}
\end{equation*}

\vspace*{-5pt}

\noindent The matrix $U := \{ u_i | i = 1 \dots n\}$ of eigenvectors of $\mc{L}$ is called the \emph{spectrum} of $\mc{L}$. It can be computed based on eigenvalue of $\mc{L}$.

% Intuitively, the Laplacian is a linear operator that, along with the spectrum (the set of vectors, $u_i \in U_S, i \in S = \{ 1,2, \ldots, n \}$), can be used to describe the structure, or topology (explained below), of a graph. To see this, consider the following equation that computes the sum of squares of edge lengths of the $i^{\textrm{th}}$ node in neighborhood $\mc{N}_i = \{ j \ | \ e_{ij} \in \mc{E} \}$:

% \begin{equation}
%      u_i^T \mc{L} u_i = \sum_{j\in \mathcal{N}_i} \brr{e_{ij}}^2
%      \label{eq: spectrum}
%     %  \vspace{-10pt}
% \end{equation}

% \noindent The lower the value of $u_i^T L u_i$, the smoother, or more ``stable", the structure of the graph. To illustrate this, consider a substance (heat, electricity) flowing between vertices through the edges of a graph, $\mc{G}$. Then, if $e_{ij}$ denotes the amount of heat (or current) that flows from vertex $v_i$ to vertex $v_j$, $\mc{L}$ measures the difference between the amount of heat at vertex $i$ and the average of the amount of heat contained in its neighboring vertices. Therefore, a lower value of $u_i^T L u_i$ indicates that the amount of heat at vertex $i$ is similar to the amount of heat at its neighbors which is characteristic of a stable system. We provide a derivation to Equation~\ref{eq: spectrum} in the supplementary material.
\subsection{Graph Topology}
\label{subsec: topology}

The topology of a vertex $v_i$ in a graph refers to the arrangement of the adjacent neighborhood $\mc{N}_i$ of the vertex. For example, if $(v_i, v_j) \in \mc{E}$ $\forall v_j \in \mc{N}_i$ and $(v_j, v_k) \notin \mc{E}$ $\forall \{v_j, v_k\} \in \mc{N}_i$, then we have a star topology~\cite{star}. The spectrum $U$ of the Laplacian matrix $\mc{L}$ can be used to compute the topology of a vertex. Let $w_i \in \mathbb{R}^{n}$ denote the $i^{\textrm{th}}$ column of the matrix $\mc{L} U$. Then the $j^{\textrm{th}}$ entry of $w_i$ is given by,
%, where $U \in \mathbb{R}^{n\times k}$ is the eigenvector matrix, that is, $w_i = \li u_i$
\vspace*{-6pt}
\begin{equation}
w_i(j) = D(j,j)u_i(j) - \sum u_i(k) = \sum \brr{u_i(j)-u_i(k)}.
\label{eq: w}
\end{equation}
\vspace*{-1pt}
\noindent Thus, the $j^{\textrm{th}}$ entry of $w_i$ equals the sum of edge values of all edges with $v_j$ as the common node. Since the arrangement of vertices depend on the edge lengths between them, this determines the topology of the graph.

\section{Behavior Classification Using \algoname}
\label{sec: algorithm}
In this section, we first show the construction of traffic graphs. Next, we present our eigenvalue algorithm. Finally, we present an analysis of our eigenvalue algorithm.

\subsection{Computation of Traffic Graphs}
\label{subsec: computation of graphs}
Traffic-graphs and their topologies form the basis for our driver behavior classification approach. Behavioral traffic psychology studies show that aggressive drivers tend to display maneuvers such as tailgating, over-speeding, overtaking, and weaving~\cite{rohanref3}. From a topological perspective, all four maneuvers are similar (Figure~\ref{fig: main}, right) with a sharper graph structure with larger edge lengths. On the other hand, conservative vehicles tend to adhere to a uniform speed, forming a smoother topology with smaller edge lengths (Figure~\ref{fig: main}, left). Thus, topologies for aggressive behavior are different from the topologies for conservative behavior. We describe the formation of traffic-graphs next.
% Behavioral traffic psychology studies show that aggressive drivers tend to display maneuvers such as tailgating, over-speeding, overtaking, and weaving~\cite{rohanref3}. From a topological perspective, all four maneuvers are similar (Figure~\ref{fig: main}, right) with a sharper graph structure with larger edge values. On the other hand, conservative vehicles tend to adhere to a uniform speed. Conservative behavior has a smoother topology with smaller edge value (Figure~\ref{fig: main}, left). Such behavior-driven maneuvers and their topologies can be represented effectively through graphs (See Section~\ref{sec: background} for construction of Adjacency and Laplacian matrices). More specifically, the spectrum of the graph can be used to represent the topologies of the graph (refer to discussion in Section~\ref{subsec: topology}), or in our application, they can be used to compute the relative positions of each road-agent in traffic at any time instant. For example, in Figure~\ref{fig: main}, the topology corresponding to aggressive drivers are generally associated with eigenvectors of large eigenvalues, while the topology corresponding to conservative drivers are generally associated with eigenvectors of smaller eigenvalues.

% \subsection{Eigenvalue Algorithm of GraphRQI}
% \vspace*{-5pt}

\begin{algorithm}
% \floatname{algorithm}{GraphRQI: }
% \resizebox{0.5}{!}{
\caption{Our eigenvalue algorithm computes the $k$ eigenvectors, $\{u_1, u_2, \ldots, u_k\}$ of a Laplacian, $\li_i$. }
\label{algo}
\begin{algorithmic}[1]
\Require $\li_0,\Lambda_0$
\For {$j$ in range$(k)$}{}
\State{Initialize $x_\textrm{old}$}
\State{$x_\textrm{new} = x_\textrm{old} / \Vts{x_\textrm{old}}$}
\While {$\Vts{x_\textrm{new}-x_\textrm{old}} \leq \epsilon$}
\State{$x_\textrm{old} = x_\textrm{new} / \Vts{x_\textrm{new}}$}
\State {$x_\textrm{new} \gets \textsc{update}( \sigma \sigma^\top,\li_{t-1}, \mu_j)$}
% \State{$x_\textrm{old} = x_\textrm{new}$}
\EndWhile
\State{$u_j = x_\textrm{new} / \Vts{x_\textrm{new}}$}
\EndFor
\Return $U = [u_1, u_2, \ldots, u_k]$

\end{algorithmic}
% }
% \vspace{-10pt}
\end{algorithm}
% \vspace*{-5pt}
% We present our algorithm, GraphRQI, in this section. The objective of GraphRQI is as follows: At a time-step $i$, given a Laplacian, $\li_t$, the goal is to compute the top $k$ eigenvectors, $u_i | i \in [k]$, that collectively comprise the spectrum, $U$. 

At the $t$-th time-step, we obtain the $k$-nearest neighbors~\cite{knn} of each road-agent using their current trajectories. In our experiments, we heuristically choose $k=4$. We initialize an adjacency matrix $A_t$ with all zeros, and assign $A(i,j) = 1$ if road-agents $v_i, v_j$ are neighbors. We then form the Laplacian matrix (ala Section~\ref{subsec: graph_spectrum}) at the $t^{\textrm{th}}$ time-step, $\li_t = D_t - A_t$. As road-agents observe  new neighbors, we update the Laplacian matrix using the following update:

\[
\li_t =
\left[
\begin{array}{c|c}
\li_{t-1} \Bstrut & 0 \Bstrut\\
\hline
0 \Tstrut & 1
\end{array}
\right] + \delta\delta^\top,
\]

\noindent where $\li_{t-1}$ is the leading principal sub-matrix of $\li_t$ of order $k$ and $\delta$ is defined in Table~\ref{tab: notation}. For our experiments, we reset the Laplacian matrix after $T=100$ to avoid buffer overflow.

% However, obtaining the eigenvector matrix, $U$, for a Laplacian $\li_t \in \mathbb{R}^{d \times d}$ is a challenging task for many realtime applications since practical numerical solvers require $\bigO{n^{2.5}}$ operations for dense matrices. If the matrix has some special structure such as symmetry or sparsity, we can reduce the computational cost. 
% Currently the fastest known algorithm (cite CG) for computing $U$ for sparse, symmetrical matrices is $\bigO{n^{3/2}}$ with a storage cost of $\bigO{n}$.
% In the next section, we give an algorithm to generate $U$ using only $\cost$ operations, where $\vts{\li^{\scriptscriptstyle -1}_i} \ll n^2$, and with storage cost of $\size$.

\subsection{Eigenvalue Algorithm}

Our eigenvalue algorithm is built upon the classical RQI algorithm~\cite{rqi} that computes an eigenvector $u$ that corresponds to an approximation of a given eigenvalue $\mu$ of a given matrix. However, the dominant step of the RQI consists of matrix inversion that generally requires $\bigO{d^3}$ operations, and so the applicability of RQI to a sequence of dynamic ( or time-varying) matrices largely depends on two factors --- computational complexity of matrix inversion, and the length of the sequence. For a sequence of dynamic Laplacian matrices,  $\{ \li_1, \li_2, \ldots , \li_T \}$, the main advantage of the GraphRQI approach is to be able to compute the eigenvector matrix, $U$, very efficiently by combining the following optimizations:

\begin{itemize}
    \item Recursively exploit sub-$k$ matrix information.
    \item Exploit the sparsity and symmetry of Laplacian matrices to compute inverse Laplacian matrices efficiently. 
\end{itemize}

We give the pseudo-code of our eigenvalue algorithm in Algorithm~\ref{algo}. At each time-step $t$, we compute $k$ eigenvectors of $\li_t$. For each eigenvector, we perform an iterative process. We begin by initializing a random vector. Next, we iteratively perform the following update rule until it converges to an eigenvector. For the $j^{\textrm{th}}$ eigenvector, the update rule is given as:
\vspace*{-5pt}
\begin{equation}
\resizebox{\columnwidth}{!}{
    $x_\textrm{new} \gets \textsc{SM}\brr{\delta\delta^\top,
\left[
\begin{array}{c|c}
\brr{1-\mu_j}\Big ( \textsc{SM}(\sigma\sigma^\top, \li_{t-1}) \Big )  \Bstrut & 0 \Bstrut \\
\hline
0 \Tstrut & 1-\mu_j
\end{array}
\right]
}x_\textrm{old}$
}
\label{eq: label}
\end{equation}
\vspace*{-1pt}
\noindent where \textsc{sm} refers to the Sherman-Morrison formula~\cite{sm} that computes the inverse of a sum of a matrix and an outer product, $x x^{\top}$, where $x\in \mathbb{R}^{d\times 1}$ is a sparse vector, $\mu_j$ is the approximate eigenvalue corresponding to which we compute an eigenvector, $u_j$. $\li_{t-1}, U_{t-1}$, and $\Lambda_{t-1}$ are the Laplacian, spectrum, and corresponding diagonal eigenvalue matrix of the previous time-step. $\delta$ is a sparse $\mathbb{R}^{d}$ vector where, if the $k^\textrm{th}$ entry of $\delta$ is denoted by $\delta(k)$, then the $k^\textrm{th}$ road-agent observes a new neighbor if $\delta(k)=1$.

\begin{table*}[h]
\caption{Ablation study on TRAF and ARGO datasets. We perform a running time analysis of several eigenvalue algorithms. All experiments were performed on an 8 Core Intel Xeon(R) W2123 CPU clocked at 3.60GHz with 32 GB RAM to compute the eigenvectors for a $d\times d$ matrix, where $d$ is the number of road-agents. We also compare the accuracy of different supervised learning machine learning models and report the weighted classification accuracy.}
% We perform two experiments -- first, we replace GraphRQI with different spectrum generating algorithms. The second experiment is to}
\centering
\resizebox{\textwidth}{!}{%
\begin{tabular}{ccccccc|cccc} 
\toprule
 & \multicolumn{6}{c}{Eigenvalue Algorithms (Runtime)} & \multicolumn{4}{c}{GraphRQI + ML model (Accuracy)}\\
\cmidrule{1-7}
\cmidrule{8-11}
Dataset \Tstrut  & \Bstrut SVD & IncrementalSVD~\cite{incsvd} & RestartSVD~\cite{zhang2018timers}  & TruncatedSVD & CG  & \textbf{GraphRQI}  & LogReg & SVM & LSTM &  \textbf{GraphRQI} \\
\midrule
% \multirow{4}{*}{TRAF }
        TRAF \Tstrut &34.2ms & 67ms &1,644ms&37ms& 2,365ms &\textbf{16.9ms}&69.1\%&74.2\%& 71.4\% & \textbf{78.3\%}  \\
%      &&&&&& \\
%               &&&&&& \\
% % \cline{2-8}
%               &&&&&& \\
% \midrule
% \multirow{4}{*}{ARGO }
        ARGO \Tstrut &16.9ms &45ms&1,091ms&35ms& 2,328ms &\textbf{10ms}& 69.9\%& 75.0\%& 70.8\%&\textbf{ 89.9\%} \\
%       &&&&&& \\
%               &&&&&& \\
% % \cline{2-8}
%              &&&&&& \\
\bottomrule
\end{tabular}
}
\label{tab:ablation}
\vspace{-10pt}
\end{table*}

\subsection{Graph Spectrum Analysis}

We compute the spectrum of the graph that describes the topologies of traffic at various time-step to classify aggressive and conservative behaviors.

In Table~\ref{tab:bigo}, we compare the time and space complexities of several \sota~eigenvalue algorithms. For each of these methods, the input is a Laplacian matrix. We now state the theoretical guarantees for the running time and storage cost of our eigenvalue algorithm in the following theorems. All proofs are provided in the supplementary material\footnote{\url{https://gamma.umd.edu/graphrqi}}.
\begin{theorem}
Given a sequence, $\{ \sg_1,\sg_2, \ldots, \sg_T\} $, our eigenvalue algorithm for $\sg_t$ converges cubically with a running time complexity of $\cost$ for each Laplacian at time-step $t$, where $k \ll \textnormal{dim}(\li_t)$, with a storage cost of $\size$.
\label{thm: 1}
\end{theorem}

% \begin{proof}

% \end{proof}

Informally, Theorem~\ref{thm: 1} states that for a Laplacian matrix $\li_t$, the compuatation for each eigenvector requires $\costk$ operations, which is fewer than quadratic runtime algorithms that require $\bigO{d^2}$ operations, since $\costk \ll d^2$. For a random intial iterate, $x$, our eigenvalue algorithm converges cubically, that is, at each iteration, $\vts{x_\textrm{new}-x_\textrm{old}} = \bigO{\epsilon^3}$. 

Note that some eigenvalue algorithms make several assumptions in terms of running time complexity analysis. For example, the conjugate gradient method~\cite{cg} requires $\bigO{d\sqrt{\kappa}}$ operations. The success of conjugate gradient assumes the matrix is either well-conditioned or pre-conditioned by a preconditioning matrix, for which the running time becomes $\bigO{d^{1.5}}$. IncrementalSVD~\cite{incsvd} requires the matrix to be low-rank. In the general case, IncrementalSVD requires $\bigO{\vts{\li_t}d}$ which becomes $\bigO{\vts{\li_t}r}$ for low rank matrices with rank $r$. Our eigenvalue algorithm makes no such assumptions. We now state a lemma, related to our Laplacian matrix, that is used to prove Theorem~\ref{thm: 1}.

\begin{lemma}
The complexity of computing the inverse of a Laplacian at time-step $t$, $\li^{\scriptscriptstyle -1}_t$, grows as $\bigO{\vts{\li^{\scriptscriptstyle -1}_t}}$.
\end{lemma}

% \begin{proof}
% ffefe
% \end{proof}

The current best-case running-time complexity for inverting a $d \times d)$ matrix is $\bigO{d^{2.373}}$ using a variation of the Coppersmith–Winograd algorithm~\cite{inversion}. We show that we can achieve an even lower runtime complexity for inverting our matrix in $\costk$, where $\costk \ll d^2$.
% \begin{theorem}
% Our eigenvalue algorithm requires a storage cost of $\size$ at time-step $i$.

% \end{theorem}

% % \begin{proof}
% % ffefe
% % \end{proof}
% Our problem statement is as follows: At any time instance, $i$, given only the trajectory of all road-agents, our goal is to predict the behavior of the driver. We begin by using the current trajectory of the road-agents to form the Laplacian matrix, $\li_t$. The next step is to use our eigenvalue algorithm to compute the spectrum, $U$. 

\subsection{Behavior Classification}

The graphs are constructed from the road-agent trajectories as described in Section~\ref{subsec: computation of graphs}. We then pass the spectrum of the traffic graphs as inputs to train a Multi-Layer Perceptron(MLP)~\cite{mlp} for behavior classification. The MLP is parametrized by a weight vector, $w$. Our goal is to learn the optimal parameters that minimize the squared error loss,
\vspace*{-5pt}
\begin{equation}
    f(w) =  \frac{1}{2} \sum_{k=1}^d \brr{w^\top y_k - z_k}^2
\end{equation}
\vspace*{-1pt}
\noindent where $d$ is the number of vehicles, $y_k$ is the $k^{\textrm{th}}$ row of the eigenvector matrix, $U$ and corresponds to the feature vector of the $k^{\textrm{th}}$ vehicle. The corresponding label of the $k^{\textrm{th}}$ vehicle is denoted as $z_k$. In our experiments, $d$ is around $100$ for each traffic video.

% \vspace{-5pt}
\section{Implementation and Performance}
% \vspace{-5pt}
We evaluate the performance of our algorithm on two open-source datasets, TRAF and Argoverse, described in Section~\ref{subsec:datasets}. In Section~\ref{subsec:metrics}, we list other classification algorithms that are used to compare~\algoname~using standard classification evaluation metrics. We report our results and analysis in Section~\ref{subsec:analysis}. We perform exhaustive ablation experiments to motivate the advantages of of~\algoname~in Section~\ref{subsec:ablation}. Finally, in Section~\ref{subsec:prediction}~we show how driver behavior knowledge can be useful and applied to road-agent trajectory prediction.
% \begin{table*}[h]
% \caption{Behavior prediction accuracy comparison on TRAF and ARGO datasets.}
% \centering
% % \resizebox{\textwidth}{!}{%
% \begin{tabular}{clccccccc} 
% \toprule
% Dataset \Tstrut & Method \Bstrut & Impatient & Reckless & Threatening & Cautious & Careful & Timid & Overall \\
% \hline
% \multirow{3}{*}{TRAF }
% & DANE~\cite{dane} \Tstrut         &&&&&&&\\
% & Cheung et al.~\cite{ernest}      &&&&&&& \\
% % & Linear Regression                &&&&&&&& \\
% % \cline{2-8}
% & \textbf{GraphRQI}                &\textbf{92.4\%}&\textbf{96.2\%}&\textbf{88.6\%}&\textbf{99.3\%}&\textbf{75.4\%}&\textbf{99.3\%}&\textbf{77.4\%} \\
% % \midrule
% \toprule
%  \Tstrut &  \Bstrut & \hbox{\transparent{0.25}Impatient} & \hbox{\transparent{0.25}Reckless} & Threatening & \hbox{\transparent{0.25}Cautious} & Careful & \hbox{\transparent{0.25}Timid} & Overall \\
% \hline
% \multirow{3}{*}{ARGO }
% & DANE~\cite{dane} \Tstrut         &&&&&&& \\
% & \hbox{\transparent{0.25}Cheung et al.~\cite{ernest}}      &&&&&&& \\
% % & Linear Regression                &&&&&&&& \\
% % \cline{2-8}
% & \textbf{GraphRQI}               &&&&&&& \\
% \bottomrule
% \end{tabular}
% % }
% \label{tab:accuracy}
% % \vspace{-10pt}
% \end{table*}

\vspace*{-5pt}

\begin{table}[h]
\caption{We report weighted classification accuracy on TRAF and ARGO datasets. Our GraphRQI considerably improves the accuracy.}
\centering
\resizebox{\columnwidth}{!}{%
\begin{tabular}{clc} 
\toprule
Dataset \Tstrut & Method \Bstrut &   Weighted Accuracy \\
\hline
\multirow{3}{*}{TRAF }
& DANE~\cite{dane} \Tstrut         & 68.1\%\\
& Cheung et al.~\cite{ernest}      & 63\% \\
% & Linear Regression                &&&&&&&& \\
% \cline{2-8}
& \textbf{GraphRQI}               &  \textbf{78.3\%} \\
\midrule
% \toprule
%  \Tstrut &  \Bstrut & Overall \\
% \hline
\multirow{3}{*}{ARGO }
& DANE~\cite{dane} \Tstrut      & 65.5\%\\
& Cheung et al.~\cite{ernest}     & 62.5\%\\
% & Linear Regression                &&&&&&&& \\
% \cline{2-8}
& \textbf{GraphRQI}              & \textbf{ 89.9\%} \\
\bottomrule
\end{tabular}
}
\label{tab:accuracy}
\vspace{-15pt}
\end{table}

% \vspace{-2pt}
\subsection{Datasets}
\label{subsec:datasets}
\paragraph{TRAF} The TRAF dataset~\cite{chandra2019traphic} consists of road-agents in dense traffic with heterogeneous road-agent captured using front and top-down viewpoints in daylight as well as night setting. These videos correspond to the highway and urban traffic in dense situations from cities in China and India.  
\paragraph{Argoverse} Argoverse~\cite{argo} contains 3D tracking annotations, 300k extracted road-agent trajectories, and semantic maps collected in urban cities, Pittsburgh and Miami.

For both datasets, we obtained behavior classification labels through crowd-sourcing and used as the ground truth. All annotators were asked to choose from the six labels. We obtain labels for about 600 agents from TRAF and 250 agents from Argoverse. More details on the labeling can be found in the supplementary material.

%We note that the current size of the dataset prohibits the applicability of deep neural networks (on account of over-fitting). We acknowledge this limitation and hope that our effort motivates the generation of more behavior-driven datasets. WELL TECHNICALLY ARGO AND OTHER PUBLIC DATASETS ARE PRETTY BIG. I DON'T KNOW IF THIS IS THE RIGHT THING TO MENTION. ONLY BECAUSE WE HAVE LABELED A FEW BEHAVIOR VIDEOS IT'S UNFAIR TO SAY CURRENT SIZE OF DATASETS LIMITS DEEP LEARNING

\subsection{Evaluation Metrics and Methods}
\label{subsec:metrics}
Our data contains a long-tail distribution of behavior labels (refer to the supplementary material for exact distribution). We report a weighted classification accuracy which defined as $A_w = \sum_{k=1}^6 f_k A_k $. Here, for the $k^{\textrm{th}}$ behavior label, $f_k$ denotes the fraction of the road-agents with that label in the ground truth and $A_k$ denotes the accuracy for that label class. We also report the running time of our eigenvalue computation measured in seconds.

We compare our approach with Dynamic Attributed Network Embedding (DANE)~\cite{dane} and Cheung et al.~\cite{ernest}. DANE considers two sets of input matrices --- the data matrix and the attribute matrix, where the data consists of the IDs of objects and attributes are features describing the objects. We use DANE for behavior classification by using the IDs of road-agents for the data matrix and their trajectories for the attribute matrix. DANE uses an eigenvalue algorithm to compute, and update, the spectrum of the Laplacians for both the data and attribute matrix. Cheung et al.~\cite{ernest} do not use an eigenvalue algorithm; instead, they use linear lasso regression based on a set of trajectory features.

%DANE performs classification on nodes based on node attributes (UNCLEAR?). In our experiments, we consider the coordinates of road-agents as attributes (WHAT ATTRIBUTES?). Cheung et al. performed driver behavior classification using linear regression on trajectory features.

%%%%%%%%%%%%%%%%%%%%%%%%%%%%%
% \begin{figure}[H]
% \begin{tabular}{ll}
% \includegraphics[width = .45\columnwidth]{figs/cat.jpg}
% &
% \includegraphics[width = .45\columnwidth]{figs/cat.jpg}
% \end{tabular}
% \caption{Confusion Matrix}
% \label{fig:confusionmatrix}
% \vspace{-10pt}
% \end{figure}
%%%%%%%%%%%%%%%%%%%%%%%%%%%%%

\vspace{-5pt}
\subsection{Analysis}
\label{subsec:analysis}
\paragraph{Classification Accuracies}We present the weighted accuracy for \algoname, DANE and Cheung et al. in Table~\ref{tab:accuracy}. We show an improvement of upto~$25$\%. 

%DANE also computes the spectrum of graphs for node classification. One key difference between DANE and our system is that DANE supports temporal changes in edge connections between a fixed set of nodes, whereas our approach does not assume a fixed set of nodes. In order to account for this discrepancy, we only pass in adjacency matrices for a particular time-step with a fixed number of road-agents. 

% \paragraph{Confusion Matrix}
\paragraph{Interpretability} Our approach of classifying the driver behavior is intuitive and interpretable. In Figure~\ref{fig:resultontraf538}, GraphRQI classifies road-agent 1 (red bounding box) as threatening and road-agent 2 (yellow bounding box) as conservative. We plot the vector, $w = \li u$, where $u$ corresponds to the largest eigenvalue of $\li$, as a gradient below the figure. The larger entries of $w$ correspond to aggressive vehicles and smaller values correspond to conservative vehicles.

\begin{figure}[t]
    \centering
    % \resizebox{.85\textwidth}{!}{
    \includegraphics[width=\columnwidth]{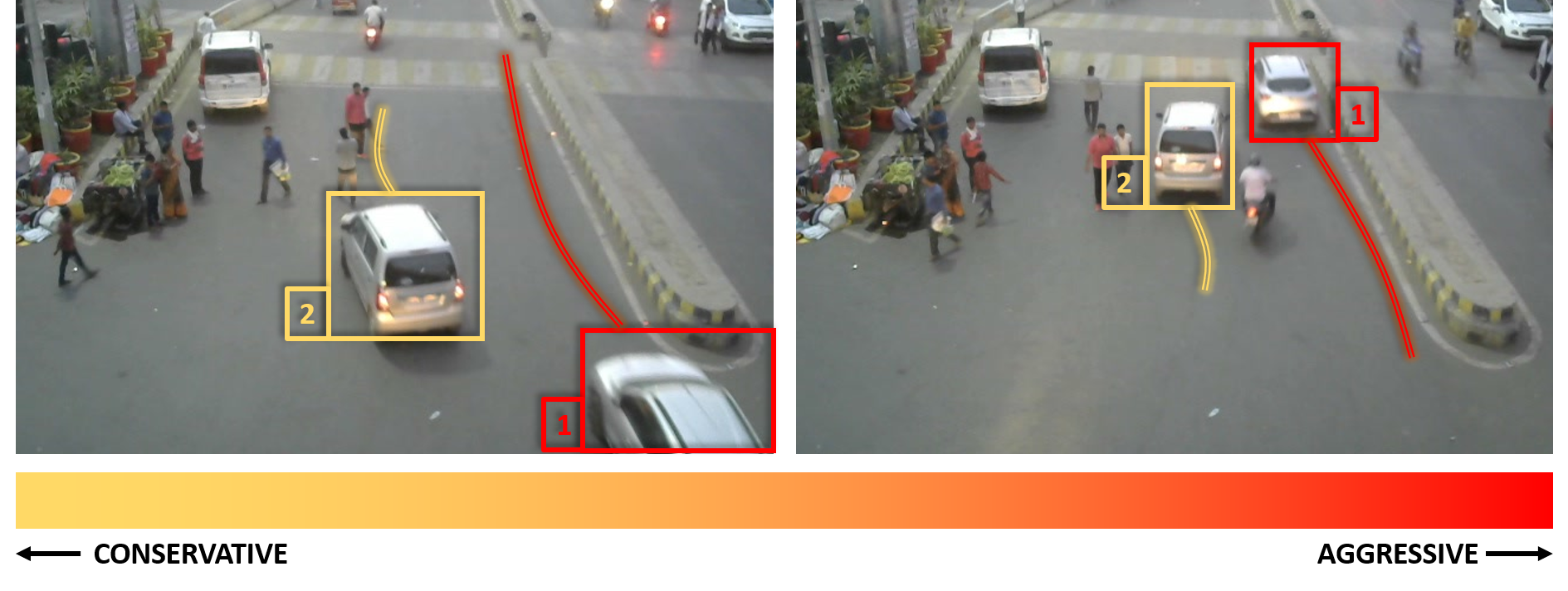}
    % }
    \caption{ \textbf{GraphRQI Performance: }GraphRQI classifies road-agent 1 (red bounding box) as threatening and road-agent 2 (yellow bounding box) as conservative. We plot the vector, $w = \li u$, where $u$ is the eigenvector corresponding to the largest eigenvalue of $\li$, as a gradient below the figure.} 
    % The larger entries of $w$ correspond to aggressive vehicles and smaller values correspond to conservative vehicles.}
    \label{fig:resultontraf538}
    \vspace{-15pt}
\end{figure}

%%%%%%%%%%%%%%%%%%%%%%%%%%%
%%%%%%%%%%%%%%%%%%%%%%%%%%%%%
\vspace{-5pt}
\subsection{Ablation Experiments}
\label{subsec:ablation}
\paragraph{Running Time Evaluation} For fair evaluation, we use the same programming platform (Python 3.7) and the same processor for running all the different methods (8 Core Intel Xeon(R) W2123 CPU clocked at 3.60GHz with 32 GB RAM). We empirically validate the theoretical guarantees of Theorem~\ref{thm: 1} by replacing our eigenvalue algorithm with several standard eigenvector algorithms, Singular Value Decomposition (SVD), Incremental SVD\cite{incsvd}, Restart SVD~\cite{zhang2018timers}, Truncated SVD, and the state-of-the-art iterative method Conjugate Gradient Descent~\cite{cg}. 
%The most prominent classical method is the singular value decomposition method which is popular to its stability. However, it is seldom used in practical applications due to its slow speed on large matrices. Several variations of SVD have been proposed to optimize speed~\cite{incsvd,zhang2018timers}. %Note that methods implemented in \textsc{python} had lower running times (measured in wall cloc) than when implemented in \textsc{matlab}. We, therefore, compare all methods in \textsc{python}.
For SVD, and TruncatedSVD, we directly used the library routine implemented in the \textsc{scipy} package, and for RestartSVD, we used the authors' original implementation. %Note that such routines are highly optimized for both commercial and academic applications. 
In practice, our method outperforms all these standard methods for Driver Behavior Classification. These results are shown in Table~\ref{tab:ablation}.%The running time of Conjugate Gradient Descent (CG) is $\bigO{d\sqrt{\kappa}}$, where $\kappa$ is the condition number of the matrix. \rohan{include analysis on cg} Note that all eigenvector algorithms converge to a set of eigenvectors, that encode the graph information of the traffic. Therefore, we note similar accuracies for all eigenvector algorithms and include the results in the supplementary material.

\vspace*{-3pt}
\paragraph{Accuracy Evaluation} GraphRQI algorithm uses MLP for the final driver behavior classification. We replace this classifier with logistic regression, support vector machines (SVM~\cite{svm})) and deep neural networks (LSTM~\cite{lstm}) and compare the results in Table~\ref{tab:ablation}. %Due to the size of the datasets used for evaluation, it is not a surprise that the LSTM network failed in comparison to shallow networks. In an attempt to understand why \rohan{Include feature analysis}

\subsection{Driver Behavior-Based  Road-Agent Trajectory Prediction}
\label{subsec:prediction}
We show that learning driver behavior can benefit road-agent navigation by improved trajectory prediction for each road-agent. Specifically, we demonstrate an application where knowledge of the behavior of a road-agent allows nearby road-agents to generate more efficient or safe paths. At any time-step $t$, given the coordinates of all road-agents and their corresponding behavior labels, we show that for the next $t + \Delta t$ seconds, aggressive (impatient, reckless, threatening) drivers may attempt to reach their goals faster through either overtaking, overspeeding, weaving, or tailgating. We also show that conservative (timid, cautious, careful) vehicles may steer away from aggressive road-agents. 
 Our application is currently limited to straight road and four-way intersection networks. When moving through the network, each vehicle’s speed is computed using a car-following model~\cite{car-following}. The vehicles are parametrized  using their positions, velocities, starting, and estimated short-term goal positions. Each vehicle is assigned a behavioral label through our behavior classification algorithm. Then as per the car-following model, the trajectories of the road agents are predicted. We include a demonstration of the application in the supplementary video.
\begin{figure}[t]
\begin{tabular}{l}
\includegraphics[width=.\columnwidth]{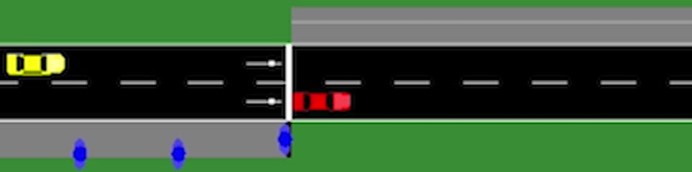}\\
\includegraphics[width=\columnwidth]{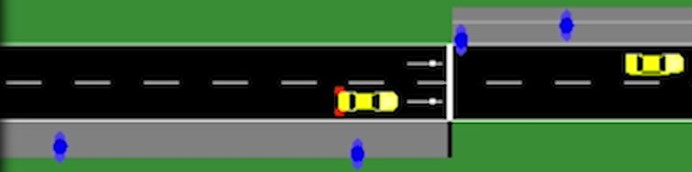}\\
% &
% \scalebox{0.8}{\includegraphics[width=0.4\textwidth]{AAAI/LaTeX/figures/face_gen.png}}
\end{tabular}
\caption{\textbf{Trajectory Prediction Application:} (\textit{top}) We see an aggressive vehicle (red) cutting across a pedestrian, not allowing it to cross first. (\textit{bottom}) Conservative vehicles, on the other hand, stop and wait for the pedestrian to cross first. }
\label{fig:sumo}
\vspace{-15pt}
\end{figure}
% \vspace{-5pt}
\section{Conclusion, Limitations, and Future Work}
% \vspace{-5pt}
We present a novel algorithm, \algoname, that can classify driver behaviors of road-agents from their trajectories. Our classification of behaviors is based on prior studies in traffic psychology, and our approach is designed for trajectory data from videos or other visual sensors. We compute the traffic matrices from the road-agent positions and use an eigenvalue algorithm to compute the spectrum of the Laplacian matrix. The spectrum comprises of the features used to train a multi-layer perceptron for driver behavior classification.  We also present a faster algorithm to compute the eigenvalue and demonstrate the performance of our algorithm on two open-source traffic datasets. We observe $25\%$ improvement in classification over prior methods.

Our approach has some limitations. The accuracy of our approach is governed by the accuracy of the tracking methods used to compute the positions of the road-agent. While the spectrum of the traffic graph can be used to classify many behaviors, it may not be sufficient to classify the behaviors in all traffic scenarios. As part of future work, we also need to evaluate our algorithms on other traffic datasets. One major issue is the lack of sufficient labeled datasets in terms of driver behaviors. 
%Our research is currently hindered by the lack of publicly available benchmark behavior-labeled autonomous driving datasets. We hope this work serves as motivation for future work in this area that spurs the development of further driver behavior datasets.
We would also like to combine the current approach with other deep learning models in order to improve accuracy and performance.
\section{Acknowledgements}

This work was supported in part by ARO Grants W911NF1910069 and W911NF1910315, Semiconductor Research Corporation (SRC), and Intel.

\bibliographystyle{IEEEtran}
\bibliography{refs}

\clearpage

\begin{appendices}
% \setstretch{1}
% \input{appendices/appendixA.tex}
% \input{appendices/appendixB.tex}
% \input{appendices/appendixC.tex}
% \input{appendices/appendixD.tex}
\end{appendices}

\appendix
\section{Proof of Theorem IV.1}
\begin{thm}
% Given a sequence, $\{ \sg_1,\sg_2, \ldots, \sg_T\} $, our eigenvalue algorithm for $\sg_t$ requires $\cost$ flops for each Laplacian at time-step $t$ to compute $k$ eigenvectors, where $k \ll \textnormal{dim}(\li_t)$
Given a sequence, $\{ \sg_1,\sg_2, \ldots, \sg_T\} $, our eigenvalue algorithm for $\sg_t$ converges cubically with a running time complexity of $\cost$ for each Laplacian at time-step $t$, where $k \ll \textnormal{dim}(\li_t)$, with a storage cost of $\size$.
% , with a storage cost of $\size$.
\label{thm: 1}
\end{thm}

\begin{proof}
Our eigenvalue algorithm is motivated from the Rayleigh Quotient Iteration algorithm~\cite{rqi} (RQI). In RQI, to compute the eigenvector of any given matrix, $A$, corresponding to an approximation to a given eigenvalue, say $\eta$, the standard update formula for the RQI method is given as $x_\textrm{new} = (A - \eta I)^{-1}x_\textrm{old}$, where $x_\textrm{old}\in \bb{R}^{d \times 1}$ is initialized as a random vector. The efficiency of the update step depends on the efficiency of computing $(A - \eta I)^{-1}$. In GraphRQI, denoting $m = \vts{\mathcal{L}_t^{\scriptscriptstyle -1} }$ as the number of non-zero elements of $\mathcal{L}_t^{\scriptscriptstyle -1} $, and $d =  \textrm{dim}\li_t$, for each of the $k$ eigenvectors, the iterative update formula is: $x_\textrm{new} = (\mathcal{L}_t - \mu I)^{-1}x_\textrm{old}$. 
% Explicitly computing this inverse requires $\bigO{d^3}$ flops in the worst case, while the best case can be achieved in $\bigO{d}$. 
At first glance, it seems as if this update can be performed by applying the Sherman-Morrison (SM) update on the following decomposition:

\begin{equation}
    (\mc{L}_t-\mu I)^{-1} = \bss{\mc{L}_t - \mu \brr{I_1 + I_2+ \ldots + I_d} }^{-1}
\label{eq: bad decomp}
\end{equation} 

\noindent That is, by expressing $\mu I$ as a sum of $d$ rank-1 matrices, $I_j$, where each $I_j$ is a diagonal matrix with $I(j,j)=1$ and $0$ elsewhere. However, SM on $\mc{L}_t$ in the above equation requires $\bigO{d^3}$ flops ($d^2$ flops for each SM operation applied $d$ times. See Lemma~\ref{lemma3}). Therefore, we look towards a better solution. 

Observe that we can decompose $\mc{L}_t = \mc{L}_t^{'} + r r^\top$, where $\mathcal{L}_t^{'} = \left[
\begin{array}{c|c}
\mathcal{L}^r_{t-1} \Bstrut & 0 \Bstrut\\
\hline
0 \Tstrut & 1
\end{array}
\right]$ and $r r^\top$ is an outer product of rank 2 with $r \in \bb{R}^{d \times 2} $ is a vector of 1's and 0's, where $r(j)=1$ represents the addition of a new neighbor by the $j^\textrm{th}$ road-agent. $\mathcal{L}^r_{t-1}$ denotes the laplacian matrix at the previous time-step plus a diagonal update on account of observing $r$ new neighbors at the current time-step (This is further explained in lemma IV.2). So,

\[ (\mc{L}_t-\mu I)^{-1} = (\mc{L}_t^{'} + r r^\top -\mu I)^{-1},\]

\noindent where $r r^\top$ is a sum of two rank-1 matrices. 
% The benefit of such a decomposition is that we now have to apply SM twice (on $\mc{L}_t^{'} -\mu I$) instead of $d$ times on $\mc{L}_t$. 
% This reduces the cost of equation~\ref{eq: bad decomp} by a factor of $d$, assuming we ignore constants.
Therefore, if we can compute $(\mc{L}_t^{'} -\mu I)^{-1}$ in less than or equal to $\bigO{m}$, we are done. It is easy to verify that for any matrix $A$, \[\left[
\begin{array}{c|c}
A \Bstrut & 0 \Bstrut\\
\hline
0 \Tstrut & 1
\end{array}
\right]^{-1} = \left[
\begin{array}{c|c}
A^{-1} \Bstrut & 0 \Bstrut\\
\hline
0 \Tstrut & 1
\end{array}
\right].\] Therefore, \[\brr{\mathcal{L}_t^{'}}^{\scriptscriptstyle -1} = \left[
\begin{array}{c|c}
\mathcal{L}^r_{t-1} \Bstrut & 0 \Bstrut\\
\hline
0 \Tstrut & 1
\end{array}
\right]^{-1} =  \left[
\begin{array}{c|c}
\brr{\mathcal{L}^r_{t-1}}^{\scriptscriptstyle -1} \Bstrut & 0 \Bstrut\\
\hline
0 \Tstrut & 1
\end{array}
\right]\]
and thus,
% $(\mathcal{L}_t^{'} -\mu I)^{-1} = U_{t-1} (\Lambda - \mu I )^{\scriptscriptstyle -1} U_{t-1}^\top$.
\[(\mathcal{L}_t^{'} -\mu I)^{-1} = \left[
\begin{array}{c|c}
\brr{\mathcal{L}^r_{t-1} - \mu I}^{\scriptscriptstyle -1} \Bstrut & 0 \Bstrut\\
\hline
0 \Tstrut & (1-\mu)
\end{array}
\right].\]
% , where $*$ denotes element-wise multiplication. 
We can now use Lemma IV.2 to show that computing the inverse of the laplacian, $\brr{\mathcal{L}^r_{t-1} - \mu I}$ can be performed in $\bigO{m}$. The entire update can be summarized as follows, 

\begin{equation}
\resizebox{\columnwidth}{!}{
    $x_\textrm{new} \gets \textsc{SM}\brr{r r^\top,
\left[
\begin{array}{c|c}
\brr{1-\mu_j}\Big ( \textsc{SM}(\sigma\sigma^\top, \li_{t-1}) \Big )  \Bstrut & 0 \Bstrut \\
\hline
0 \Tstrut & 1-\mu_j
\end{array}
\right]
}x_\textrm{old}$
}
\label{eq: label}
\end{equation}

\noindent $\sigma\sigma^\top = \Sigma$.
% , again, assuming ignorance of constants.

At each time-step, we add a new column and a row to the laplacian matrix of the previous time-step. Therefore, the space complexity grows linearly with the dimension of the laplacian matrix. Finally, the cubic convergence follows from the classic RQI algorithm~\cite{rqi}.

\end{proof}
\section{Proof of Lemma IV.2}

\begin{lem}
% We can compute $\brr{\mathcal{L}^r_{t-1} - \mu I }^{\scriptscriptstyle -1}$ in $\bigO{m}$.
The complexity of computing the inverse of a Laplacian at time-step $t$ grows as $\bigO{m}$.
% , with a storage cost of $\size$.
\label{thm: 1}
\end{lem}

\begin{proof}

Note that $\mathcal{L}^r_{t-1} \neq \mathcal{L}_{t-1} \implies \brr{\mathcal{L}^r_{t-1}}^{\scriptscriptstyle -1} \neq \mathcal{L}_{t-1}^{\scriptscriptstyle -1}$ since the diagonal elements (the degree of each node (vehicle)) are now different due the entry of new nodes at current time $t$. Let the number of new nodes be $r$. For example, if at the current time-step, 10 cars from the previous time-step observed a new neighbor, then $r=10$. Since we heuristically set the time-step and the KNN parameters, we constrain each vehicle observing at most 1 new neighbor for each time-step. Therefore, 

\[ \mathcal{L}^r_{t-1} = \mathcal{L}_{t-1} + \Sigma, \]

\noindent where $\Sigma =  I_1 + I_2 + \ldots + I_r$ is a low-rank matrix with rank $r \ll d$, where we treat $r$ as a constant due to our earlier assumption. Each $I_j$ is a diagonal matrix with $I(j,j)=1$ and $0$ elsewhere. Hence, 

\begin{equation}
\begin{split}
    \brr{\mathcal{L}^r_{t-1} - \mu I} &= \mathcal{L}_{t-1}  - \mu I + \Sigma \\
    \brr{\mathcal{L}^r_{t-1} - \mu I }^{\scriptscriptstyle -1} &= \brr{\mathcal{L}_{t-1} - \mu I + I_1 + I_2+ \ldots + I_r}^{-1},\\ 
\end{split}
\end{equation}

\noindent where $(\mathcal{L}_{t-1} -\mu I)^{-1} = U_{t-1} (\Lambda - \mu I )^{\scriptscriptstyle -1} U_{t-1}^\top$. Using the Sherman-Morrison formula $r$ times on $(\mathcal{L}_{t-1}- \mu I)$, we get $\brr{\mathcal{L}^r_{t-1}-\mu I}^{\scriptscriptstyle -1}$ in $\bigO{m}$ (See appendix~\ref{lemma3}).

\end{proof}

\section{Lemma 3: Each Sherman-Morrison operation requires $\bigO{m}$ flops.}
\label{lemma3}
\begin{proof}

For a matrix of size $d \times d$, each SM operation requires,

\begin{itemize}
    \item 2 matrix-vector multiplications - $\bigO{m}$
    \item 1 inner product - $\bigO{m}$
    \item 1 outer-product computation - $\bigO{m}$
    \item 1 matrix subtraction - $\bigO{m}$
\end{itemize}

\end{proof}
\section{Annotation of TRAF and Argoverse Datasets}

We study and classify road-agents into six different driving behaviors. We used two datasets for training a neural network and the class-wise distribution of the final labels are shown in Figure~\ref{fig:annot}.

\begin{table}[h]
    \centering
\resizebox{0.8\columnwidth}{!}{
\begin{tabular}{lc} 
\toprule
\textbf{Label} \Tstrut & \textbf{Traits} \Bstrut  \\
\hline
\multirow{4}{*}{Impatient}\Tstrut         & Overtaking\\
                & Over-speeding\\
 & Tailgating\\
                & Weaving\\
                \midrule
\multirow{2}{*}{Threatening}      & Cutting in front of agents \\
                 & Driving extremely close to agents \\
                 \midrule
% & Linear Regression                &&&&&&&& \\
% \cline{2-8}
 \multirow{3}{*}{Reckless}                &  Driving in the wrong direction\\
                         & Jaywalking through traffic\\
                         & Not following rules \\
\midrule
% \toprule
%  \Tstrut &  \Bstrut & Overall \\
% \hline
\multirow{3}{*}{Careful}                & Maintain neighbors\\
                                        &Uniform Speed\\
                                        &Follow rules\\
\midrule

 \multirow{2}{*}{Timid}     & Not moving\\
            &Slowing down/coming to stop\\
% & Linear Regression                &&&&&&&& \\
\midrule
% \cline{2-8}
 Cautious              & Waiting for vehicles/pedestrians \\
\bottomrule
\end{tabular}
}
\caption{Driver Behavior Traits.}
\label{tab:annot}
\end{table}

% \vspace{15pt}

Studies in traffic psychology show that each of the six driving behaviors are associated with certain driving traits. For example, impatient drivers show different traits (e.g., over-speeding, overtaking) than reckless drivers (e.g., driving in the opposite direction to the traffic)~\cite{rohanref6-reckless}. Similarly, timid drivers may drive slower than the speed limit~\cite{rohanref7-timid}, and careful drivers obey the traffic rules. We now describe the labeling process. First we distinguish between annotators and expert annotators. Annotators correspond to the individuals that labeled the videos via crowd-sourcing. They may or may not have expertise or familiarity with the traffic in the videos. All annotators were asked to observe these traits in the videos and list the most prominent traits observed in each road-agent. The trait table is shown in the Table~\ref{tab:annot}.

The annotators were not shown the class-wise distribution of traits, nor were they shown annotations marked by any other annotator. Each annotator was presented with a uniform presentation of traits and their task was to simply mark the traits that they observed for each vehicle. Each annotator labeled every video. The final class was assigned by resolving discrepancies for each road-agent. Some vehicles would receive conflicting traits. If the conflicting traits belonged to the same Parent class (aggressive or conservative), then we would simply select the majority trait. In rare scenarios where there is no majority, we would simply take a final call. However, if the conflicting traits belonged to different parent classes, then that road-agent would be re-annotated through crowd-sourcing, and would repeat the steps.
    \begin{figure}[H]
    \centering
    \includegraphics[width=\columnwidth]{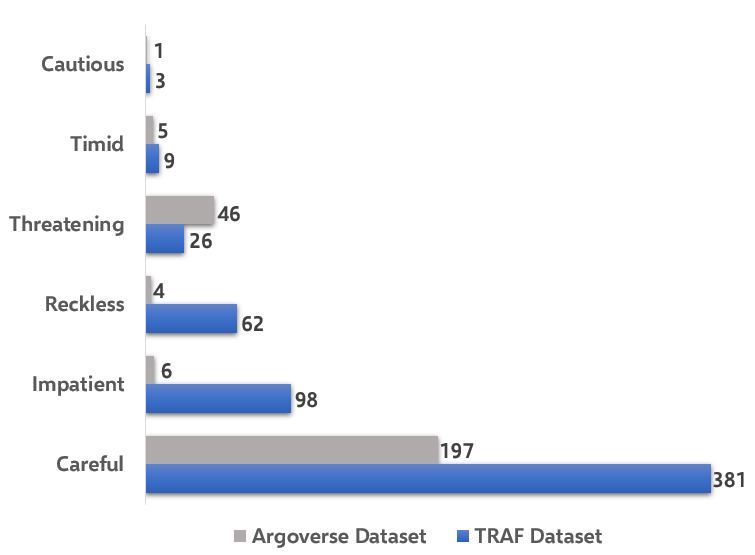}
    \caption{Class-wise distribution of the final labels}
    \label{fig:annot}
\end{figure}

\end{document}